\documentclass{llncs}
\usepackage{amsmath}
\usepackage{amssymb}
\usepackage{times}
\usepackage{indentfirst}
\usepackage{graphicx}
\usepackage{fancyhdr}
\usepackage{epstopdf}
\usepackage{epsfig}
\usepackage{subfigure}

\begin{document}

\title{Connectivity for matroids based on rough sets}

\author{Bin Yang and William Zhu~~\thanks{Corresponding author.
E-mail: williamfengzhu@gmail.com (William Zhu)}}
\institute{
Lab of Granular Computing,\\
Minnan Normal University, Zhangzhou 363000, China}



\date{\today}          
\maketitle

\begin{abstract}
In mathematics and computer science, connectivity is one of the basic concepts of matroid theory: it asks for the minimum number of elements which need to be removed to disconnect the remaining nodes from each other. It is closely related to the theory of network flow problems. The connectivity of a matroid is an important measure of its robustness as a network. Therefore, it is very necessary to investigate the conditions under which a matroid is connected.
In this paper, the connectivity for matroids is studied through relation-based rough sets. First, a symmetric and transitive relation is introduced from a general matroid and its properties are explored from the viewpoint of matroids. Moreover, through the relation introduced by a general matroid, an undirected graph is generalized. Specifically, the connection of the graph can be investigated by the relation-based rough sets. Second, we study the connectivity for matroids by means of relation-based rough sets and some conditions under which a general matroid is connected are presented. Finally, it is easy to prove that the connectivity for a general matroid with some special properties and its induced undirected graph is equivalent.
These results show an important application of relation-based rough sets to matroids.\\                                                                                                                                                                                                   


\end{abstract}


\section{Introduction}

Rough set theory, proposed by Pawlak~\cite{Pawlak82Rough,Pawlak91Rough} in 1982, is a useful
tool for dealing with the vagueness and granularity in information systems. With more than twenty years development, rough set theory
has attracted much research interest in resent years. In theory, it has been extend to generalized rough sets based on reflexive relations~\cite{Kondo05OnTheStructure,QinYangPei08Generalized}, on similarity relations~\cite{SlowinskiVanderpooten00AGeneralized}, on tolerance
relations~\cite{BartolMiroPioroRossello05OntheCoverings,YaoWangZhang10Onfuzzy,SkowronStepaniuk96tolerance}, on arbitrary relations~\cite{Diker10Textural,LiuZhu08TheAlgebraic,Zhu07Generalized}, covering-based rough sets~\cite{ZhuWang03Reduction,Zhu07Topological,Zhu09RelationshipAmong,Zhu09RelationshipBetween}, probabilistic rough sets~\cite{Yao10Three-way,Yao11TheSuperiority}, fuzzy rough sets~\cite{DengChenXuDai07ANovel,GongSunChen08Rough} and so on. In applications, it has
been successfully applied in knowledge discovery~\cite{Liu10Rough}, machine learning~\cite{ChenMiaoWangWu11ARough}, knowledge acquisition~\cite{LeungWuZhang06Knowledge,QianLiangYaoDang,WeiLiZhang07Knowledge} and decision analysis~\cite{Huyuguo10Fuzzypreferencebased,MinLiu09Ahierarchical,QianLiangLiZhangDang08Measures,QianLiangPedryczDang10Positive,XuZhao09Amulti-objective}.

Generalized rough sets based on symmetric and transitive relations are a generalization of classical rough sets. In many real word applications,
this generalization is quite useful. An example is that data in incomplete information/decision systems often generate symmetric and transitive relations~\cite{QianLiangDang10Incomplete}. Therefore, there is profound theoretical and practical significance to study this type of relation-based rough sets.

The concept of matroids was
originally introduced by Whitney~\cite{Whitney35Onthe} in 1935 as a generalization of graph theory and linear algebra.
Matroid theory is a branch of combinatorial mathematics. It is widely used in optimization. Therefore, it is a good idea to
integrate rough set and matroid. Some interesting results about the connection between matroids and rough sets can be found in literatures~\cite{LiLiu12Matroidal,LiuZhu12matroidal,TangSheZhu12matroidal,Wangzhu13Applicationsofmatrices,WangZhuZhuMin12Matroidal,YangZhu13Matroidal,YangZhu13covering-basedrough}.

Connectivity is one of the basic concepts of matroid theory: it asks for the minimum number of elements which need to be removed to disconnect the remaining nodes from each other. It is closely related to the theory of network flow problems. The connectivity of a matroid is an important measure of its robustness as a network. Therefore, it is very necessary to investigate the conditions under which a matroid is connected.

In this paper, whether a general matroid is connected or disconnected is studied by means of relation-based rough sets. For this purpose, a symmetric and transitive relation $R(M)$ is firstly generalized from a general matroid $M$. In fact, if $\bigcup\mathcal{C}(M)=U(M)$, then $R(M)$ is also a reflexive relation. Based on the above works, some properties of $R(M)$ are presented from the viewpoint of matroid.  In~\cite{ChenLiLin13Computing},  Li et al. have studied the problem of computing connected components of a simple undirected graph through relation-based rough sets. Then we can generalize an undirected graph $G(M)$ from $R(M)$. Similarly, the connectivity of $G(M)$ can be investigated by the relation-based rough sets.
Second,  some conditions under which a general matroid is connected are presented from the viewpoint of relation-based rough set. In other words, the connectivity  for matroids can be studied through relation-based rough sets. Finally, it is easy
to prove that the connectivity for a general matroid $M$ and $G(M)$ is
equivalent if $\bigcup\mathcal{C}(M)=U(M)$.  In a word, these results show many potential connections between
relation-based rough sets and matroids.

The remainder of this paper is organized as follows: In Section~\ref{section1}, some basic concepts and properties related to relation-based rough sets and matroids are introduced. In Section~\ref{section2}, a symmetric and transitive relation is introduced from a general matroid and its properties are explored from the viewpoint of matroids. Moreover, through the relation introduced by a general matroid, an undirected graph
is generalized. Specifically, the connection of the graph can be investigated by
the relation-based rough sets. In Section~\ref{section3}, some conditions under which a general matroid is connected are presented by means of relation-based rough sets. Section~\ref{section4} concludes this paper.
\section{Preliminaries}
\label{section1}
In this section, we review some fundamental definitions and
results of relation-based rough sets and matroids.

\subsection{Relation-based rough sets}
Let $U$, the universe of discourse, be a non-empty finite set. We use $P(U)$ to denote the power set of $U$ and $X^{C}$ to denote the complement of $X$ in $U$.

Binary relations play an important role in the theory of rough sets. In this subsection, we present some definitions and properties of relation-based rough sets used in this paper. For detailed
descriptions and proofs of them, please refer to~\cite{Yao98Constructive,Zhu07Generalized,Zhu09RelationshipBetween}.

\begin{definition}(Binary relation~\cite{RajagopalMason92Discrete})
\label{relation}
Let $U$ be a set, $U\times U$ the
product set of $U$ and $U$. Any subset $R$ of $U\times U$ is called a
binary relation on $U$. For any $(x, y)\in U\times U$, if $(x, y)\in R$,
then we say $x$ has relation with $y$, and denote this relationship
as $xRy$.

For any $x\in U$, we call the set $\{y\in U: xRy\}$ the successor
neighborhood of $x$ in $R$ and denote it as $RN(x)$.
\end{definition}

Throughout this paper, a binary relation is simply called a
relation. The
relation and its properties play an important role in studying
relation-based rough sets.
\begin{definition}(\cite{RajagopalMason92Discrete})
A relation $R$ on $U$ is called
\begin{flushleft}
(1) serial, if for all $x\in U$, there is $y\in U$ such that $xRy$;\\
(2) reflexive, if $xRx$ for all $x\in U$;\\
(3) symmetric, if $xRy$ implies $yRx$ for all $x, y\in U$;\\
(4) transitive, if $xRy, yRz$ imply $xRz$ for all $x, y, z\in U$.
\end{flushleft}
\end{definition}

We call the pair $(U, R)$ a relation-based approximation space if $R$ is a binary relation on $U$ without any additional constraints~\cite{Yao96Two}.
\begin{definition}(\cite{Yao98Constructive})
Let $(U, R)$ be a relation-based approximation space. For any $X\subseteq U$, the lower and the upper approximations are defined, respectively, by:
\begin{center}
$\underline{R}(X)=\{x\in U: RN(x)\subseteq X\}$, $\overline{R}(X)=\{x\in U: RN(x)\bigcap X\neq\emptyset\}$.
\end{center}
The operators $\underline{R}, \overline{R}: P(U)\rightarrow P(U)$ are,
respectively, called the lower and upper approximation operators
in $(U, R)$.
\end{definition}

The pair of approximation operators is related to the pair of
modal operators in modal logic~\cite{Yao98Constructive}. Then one can easily obtain
the following properties of approximation operators.
\begin{theorem}(\cite{Yao98Constructive})
Let $R$ be a relation-based approximation space. For subsets $X, Y\subseteq U$:
\begin{flushleft}
(1) $\underline{R}(X^{C})=(\overline{R}(X))^{C}, \overline{R}(X^{C})=(\underline{R}(X))^{C}$;\\
(2) $\underline{R}(U)=U, \overline{R}(\emptyset)=\emptyset$;\\
(3) If $X\subseteq Y$, then $\underline{R}(X)\subseteq\underline{R}(Y), \overline{R}(X)\subseteq\overline{R}(Y)$;\\
(4) $\underline{R}(X\bigcap Y)=\underline{R}(X)\bigcap\underline{R}(Y), \overline{R}(X\bigcup Y)=\overline{R}(X)\bigcup\overline{R}(Y)$.
\end{flushleft}
\end{theorem}

\subsection{Matroids}
Matroid theory was established as a generalization, or a
connection, of graph theory and linear algebra. This theory
was used to study abstract relations on a subset, and it uses
both of these areas of mathematics for its motivation, its
basic examples, and its notation. With the rapid development
in recent years, matroid theory has been
applied to a variety of fields such as combinatorial optimization~\cite{Lawler01Combinatorialoptimization} and greedy algorithm design~\cite{Edmonds71Matroids}. One of main characteristic of matroids is that there are many equivalent ways to define them, which
is the basis for its powerful axiomatic system. The following definition presents a widely used axiomatization on matroids.
In this subsection, we present definitions, examples and results
of matroids used in this paper.
\begin{definition}(Matroid~\cite{Lai01Matroid,LiuChen94Matroid})
A matroid is an ordered pair $M=(U, \mathcal{I})$, where $U$ is a finite set, and $\mathcal{I}$  a family of subsets of $U$ with the following three properties:\\
(I1) $\emptyset\in \mathcal{I}$;\\
(I2) If $I\in \mathcal{I}$, and $I'\subseteq I$, then $I'\in\mathcal{I}$;\\
(I3) If $I_{1}$, $I_{2}\in \mathcal{I}$, and $|I_{1}|<|I_{2}|$, then there exists $e\in I_{2}-I_{1}$ such that $I_{1}\bigcup\{e\}\in\mathcal{I}$, where $|I|$ denotes the cardinality of $I$.\\
Any element of $\mathcal{I}$ is called an independent set.
\end{definition}

\begin{example}
\label{e1}
Let $G=(V$, $U)$ be the graph as shown in Fig.\ref{fig1}. Denote $\mathcal{I}=\{I\subseteq U\mid I$ does not contain
a cycle of $G\}$, i.e., $\mathcal{I}=\{\emptyset$, $\{a_{1}\}$, $\{a_{2}\}$, $\{a_{3}\}$, $\{a_{4}\}$, $\{a_{1}$, $a_{2}\}$, $\{a_{1}$, $a_{3}\}$, $\{a_{1}$, $a_{4}\}$, $\{a_{2}$, $a_{3}\}$, $\{a_{2}$, $a_{4}\}$, $\{a_{3}$, $a_{4}\}$, $\{a_{1}$, $a_{2}$, $a_{4}\}$, $\{a_{1}$, $a_{3}$, $a_{4}\}$, $\{a_{2}$, $a_{3}$, $a_{4}\}\}$. Then $M=(U$, $\mathcal{I})$ is a matroid, where $U=\{a_{1}$, $a_{2}$, $a_{3}$, $a_{4}, a_{5}\}$.
\end{example}

\begin{figure}
\begin{center}
  \includegraphics[width=4cm]{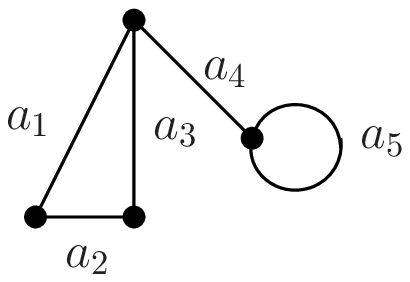}\\
  \caption{A graph}\label{fig1}
\end{center}
\end{figure}
In fact, if a subset is not an independent set, then it is a
dependent set of the matroid. In other words, the dependent
set of a matroid generalizes the cycle of graphs. Based on
the dependent set, we introduce other concepts of a matroid. For this purpose, several denotations are presented in the
following definition.
\begin{definition}(\cite{Lai01Matroid})
Let $\mathcal{A}$ be a family of subsets of $U$. One can denote

$Max(\mathcal{A})=\{X\in \mathcal{A}: \forall Y\in \mathcal{A}$, $X\subseteq Y\Rightarrow X=Y\}$,

$Min(\mathcal{A})=\{X\in \mathcal{A}: \forall Y\in \mathcal{A}$, $Y\subseteq X\Rightarrow X=Y\}$,

$Opp(\mathcal{A})=\{X\subseteq U: X\notin \mathcal{A}\}$,

$Com(\mathcal{A})=\{X\subseteq U: X^{C}\in \mathcal{A}\}$.
\end{definition}

The dependent set of a matroid generalizes the linear dependence in vector spaces and the circle in graph theory.
The circuit of a matroid is a minimal dependent set.

\begin{definition}(Circuit~\cite{Lai01Matroid})
\label{circuit}
Let $M=(U$, $\mathcal{I})$ be a matroid. A minimal dependent set in $M$ is called a
circuit of $M$, and we denote the family of all circuits of $M$ by $\mathcal{C}(M)$, i.e., $\mathcal{C}(M)=Min(Opp(\mathcal{\mathcal{I}}))$.
\end{definition}

In order to illustrate the circuit of a matroid is an extension of the cycle of a graph, the following example is presented.
\begin{example}(Continued from Example~\ref{e1})
\label{e2}
$\mathcal{C}(M)=\{\{a_{5}\}, \{a_{1}, a_{2}, a_{3}\}\}$.
\end{example}

The above example shows that the circuit of a matroid coincides with the cycle of a graph when the matroid is induced by the graph.

A matroid uniquely determines its circuits, and vice versa. The following theorem indicates that a matroid can be defined from the viewpoint of circuits.
\begin{theorem}(Circuit axiom~\cite{Lai01Matroid})
Let $\mathcal{C}$ be a family of subsets of $U$. Then there exists $M=(U, \mathcal{I})$
such that $\mathcal{C}=\mathcal{C}(M)$ if and only if $\mathcal{C}$ satisfies the following three conditions:\\
(C1) $\emptyset\notin\mathcal{C}$;\\
(C2) If $C_{1}, C_{2}\in\mathcal{C}$ and $C_{1}\subseteq C_{2}$, then $C_{1}=C_{2}$;\\
(C3) If $C_{1}, C_{2}\in\mathcal{C}, C_{1}\neq C_{2}$ and $x\in C_{1}\bigcap C_{2}$, then there exists $C_{3}\in\mathcal{C}$ such that $C_{3}\subseteq C_{1}\bigcup C_{2}-\{x\}$.

\end{theorem}

Based on the circuits of a matroid, one can define the closure of any subset of the ground set as follows. In fact, the closure
captures the essence of spanning space in linear algebra.
\begin{definition}(Closure~\cite{Lai01Matroid})
\label{closure}
Let $M=(U$, $\mathcal{I})$ be a matroid. For all $X\subseteq U, cl_{M}(X)=X\bigcup\{x\in U: \exists C\in \mathcal{C}(M)~s.t., x\in C\subseteq X\bigcup\{x\}\}$ is called the closure of $X$ with respect to $M$, and $cl_{M}$ is called the closure operator.
\end{definition}

In fact, from the viewpoint of closure operator, an axiom of matroids can also be constructed. In other words, a matroid and its
closure operator are determined by each other.
\begin{theorem}(Closure axiom~\cite{Lai01Matroid})
Let $cl: P(U)\rightarrow P(U)$ be an operator. Then there exists $M=(U, \mathcal{I})$ such that $cl=cl_{M}$ if and only if $cl$ satisfies
the following four conditions:\\
(CL1) For all $X\subseteq U, X\subseteq cl(X)$;\\
(CL2) For all $X\subseteq Y\subseteq U, cl(X)\subseteq cl(Y)$;\\
(CL3) For all $X\subseteq U, cl(cl(X))=cl(X)$;\\
(CL4) For all $x, y\in U$ and $X\subseteq U$, if $y\in cl(X\bigcup\{x\})-cl(X)$, then $x\in cl(X\bigcup\{y\})$.
\end{theorem}

\begin{definition}(Rank function~\cite{Lai01Matroid})
\label{rank}
Let $M=(U$, $\mathcal{I})$ be a matroid. We can define the rank function of $M$ as follows: for all $X\in P(U)$
\begin{center}
$r_{M}(X)=max\{|I|: I\subseteq X$, $I\in\mathcal{I}\}$.
\end{center}
We call $r_{M}(X)$ the rank of $X$ in $M$.
\end{definition}

Based on the rank function of a matroid, one can present the an equivalent formulation of closure operator, which reflects the dependency between
a set and elements.

\begin{theorem}(\cite{Lai01Matroid})
\label{closure2}
Let $M=(U$, $\mathcal{I})$ be a matroid. Then
\begin{center}
$cl_{M}(X)=\{e\in U: r_{M}(X\bigcup\{e\})=r_{M}(X)\}$ for all $X\in P(U)$.
\end{center}
\end{theorem}

\section{Binary relations, matroids and graphs}
\label{section2}

A matroid $M(U, \mathcal{I})$ is said to be a free matroid if $\mathcal{I}=P(U)$, and otherwise we say it is a general matroid.
In this section, a new relation is induced by a general matroid and its properties are investigated from the viewpoint of matroid theory.
\begin{definition}
\label{relation-matroid}
Let $M(U, \mathcal{I})$ be a general matroid. We can define
a relation $R(M)$ as follows: for all $x, y\in U$,
\begin{center}
$xR(M)y\Leftrightarrow \exists C\in \mathcal{C}(M)\rightarrow\{x, y\}\subseteq C$.
\end{center}
We say that the relation $R(M)$ is induced from $M$.
\end{definition}

According to the above definition, it is clearly that $R(M)N(x)=\bigcup\{C\in\mathcal{C}(M): x\in C\}$ for all $x\in U$.

\begin{example}
\label{e4}
Let $U=\{a_{1}, a_{2}, a_{3}, a_{4}, a_{5}, a_{6}, a_{7}\}$ where $a_{1}=[1~0~0]^{T}, a_{2}=[0~1~0]^{T}, \\a_{3}=[0~0~1]^{T}, a_{4}=[1~1~0]^{T}, a_{5}=[0~1~1]^{T}, a_{6}=[0~1~1]^{T}, a_{7}=[0~0~0]^{T}$. Denote $\mathcal{I}=\{X\subseteq U: X$ are linearly independent $\}$. Then $\mathcal{I}$
consists of all subsets of $U-\{a_{7}\}$ with at most three elements except for $\{a_{1}, a_{2}, a_{4}\}, \{a_{2}, a_{3}, a_{5}\}, \{a_{2}, a_{3}, a_{6}\}$, and any subset containing $\{a_{5}, a_{6}\}$. The pair $M=(U, \mathcal{I})$ is a particular example of a matroid. It is easy to know
$\mathcal{C}(M)=\{\{a_{7}\}, \{a_{5}, a_{6}\}, \{a_{1}, a_{2}, a_{4}\}, \{a_{2}, a_{3}, a_{5}\}, \{a_{2}, a_{3}, a_{6}\}, \{a_{1}, a_{3}, a_{4},\\ a_{5}\}, \{a_{1}, a_{3}, a_{4}, a_{6}\}\}$ by Definition~\ref{circuit}. Then the relation induced from $M$ is $R(M)=\{a_{1}, a_{2}, a_{3}, a_{4}, a_{5}, a_{6}\}\times\{a_{1}, a_{2}, a_{3}, a_{4}, a_{5}, a_{6}\}\bigcup\{(a_{7}, a_{7})\}$.

And we have
$R(M)N(a_{1})=\bigcup\{\{a_{1}, a_{2}, a_{4}\}, \{a_{1}, a_{3}, a_{4}, a_{5}\}, \{a_{1}, a_{3}, a_{4}, a_{6}\}\}=\{a_{1}, a_{2}, a_{3}, a_{4}, a_{5}, a_{6}\}$;
$R(M)N(a_{2})=\bigcup\{\{a_{1}, a_{2}, a_{4}\}, \{a_{2}, a_{3}, a_{5}\}, \{a_{2}, a_{3}, a_{6}\}\}=\{a_{1}, a_{2}, a_{3}, a_{4}, a_{5}, a_{6}\}$;
$R(M)N(a_{3})=\bigcup\{\{a_{2}, a_{3}, a_{5}\}, \{a_{2}, a_{3}, a_{6}\}, \{a_{1}, a_{3}, a_{4}, a_{5}\},\\ \{a_{1}, a_{3}, a_{4}, a_{6}\}\}=\{a_{1}, a_{2}, a_{3}, a_{4}, a_{5}, a_{6}\}$; $R(M)N(a_{4})=\bigcup\{\{a_{1}, a_{2}, a_{4}\}, \{a_{1}, a_{3}, \\a_{4}, a_{5}\}, \{a_{1}, a_{3}, a_{4}, a_{6}\}\}=\{a_{1}, a_{2}, a_{3}, a_{4}, a_{5}, a_{6}\}$; $R(M)N(a_{5})=\bigcup\{\{a_{5}, a_{6}\}, \{a_{2},\\ a_{3}, a_{5}\},\{a_{1}, a_{3}, a_{4}, a_{5}\}\}=\{a_{1}, a_{2}, a_{3}, a_{4}, a_{5}, a_{6}\}$; $R(M)N(a_{6})=\bigcup\{\{a_{5}, a_{6}\}, \{a_{2},\\ a_{3}, a_{6}\},\{a_{1}, a_{3}, a_{4}, a_{6}\}\}=\{a_{1}, a_{2}, a_{3}, a_{4}, a_{5}, a_{6}\}$; $R(M)N(a_{7})=\bigcup\{\{a_{7}\}\}=\{a_{7}\}$.
\end{example}

It is clearly $R(M)$ is a symmetric and transitive relation. In order to prove this result, the following lemma is presented.
\begin{lemma}(~\cite{Lai01Matroid})
Let $M(U, \mathcal{I})$ be a matroid. If $C_{1}, C_{2}\in\mathcal{C}(M), e_{1}\in C_{1}-C_{2}, e_{2}\in C_{2}-C_{1}$ and $C_{1}\bigcap C_{2}\neq \emptyset$, then there exists $C_{3}\in\mathcal{C}(M)$ such that $e_{1}, e_{2}\in C_{3}\subseteq C_{1}\bigcup C_{2}$.
\end{lemma}

The above lemma shows that the circuits of a matroid satisfy transitivity. Based on this lemma, the following proposition
is presented and proved.
\begin{proposition}
\label{p1}
Let $M(U, \mathcal{I})$ be a general matroid. Then $R(M)$ is a symmetric and transitive relation on $U$.
\end{proposition}
\begin{proof}
(1) symmetric. For all $x, y\in U$, if $(x, y)\in R(M)$, then there exists $C\in \mathcal{C}(M)$ such that $\{x, y\}\subseteq C$. Since
$\{x, y\}=\{y, x\}$, then $(y, x)\in R(M)$ and $R(M)$ is a symmetric relation on $U$. (2) transitive. For all $x, y, z\in U$, if $(x, y)\in R(M)$
and $(y, z)\in R(M)$, then there exist $C_{1}, C_{2}\in\mathcal{C}(M)$ such that $\{x, y\}\subseteq C_{1}$ and $\{y, z\}\subseteq C_{2}$. If $z\in C_{1}$ or $x\in C_{2}$, then $(x, z)\in R(M)$, i.e., $R(M)$ is transitive. If $z\notin C_{1}$ and $x\notin C_{2}$, then there exists $C_{3}\in C(M)$
such that $x, z\in C_{3}$, i.e., $R(M)$ is transitive.

To sum up, we have already finished the proof of this proposition.
\end{proof}

\begin{example}(Continued from Example~\ref{e1})
\label{e5}
The relation induced from $M$ in Example~\ref{e1} is $R(M)=\{(a_{1}, a_{1}), (a_{1}, a_{2}), (a_{1}, a_{3}), (a_{2}, a_{1}), (a_{2}, a_{2}), (a_{2}, a_{3}), (a_{3}, a_{1}), (a_{3}, a_{2}), (a_{3},\\ a_{3}), (a_{5}, a_{5})\}$. It is clearly that $R(M)$ is a symmetric and transitive relation on $U=\{a_{1}, a_{2}, a_{3}, a_{4}, a_{5}\}$.
\end{example}

\begin{proposition}
\label{p2}
Let $M(U, \mathcal{I})$ be a general matroid. Then
\begin{center}
$\overline{R(M)}(X)=\bigcup\{C\in\mathcal{C}(M): C\bigcap X\neq\emptyset\}$ for all $X\subseteq U$.
\end{center}
\end{proposition}

\begin{proof}
Since $\overline{R(M)}(X)=\{x\in U: R(M)N(x)\bigcap X\neq\emptyset\}$ and $R(M)N(x)=\bigcup\{C\in\mathcal{C}(M): x\in C\}$ for all $x\in U$, then
$\overline{R(M)}(X)=\bigcup\{C\in\mathcal{C}(M): C\bigcap X\neq\emptyset\}$.
\end{proof}

According to the above proposition, the widely used relation-based upper approximation operators are represented by
the circuits of matroids.
\begin{example}(Continued from Example~\ref{e1})
\label{e6}
We have
$R(M)N(a_{1})=R(M)N(a_{2})=R(M)N(a_{3})=\{a_{1}, a_{2}, a_{3}\}$;
$R(M)N(a_{4})=\emptyset$; $R(M)N(a_{5})=\{a_{5}\}$.
Let $X=\{a_{1}, a_{4}\}$ and $Y=\{a_{2}, a_{4}, a_{5}\}$. Then $\overline{R(M)}(X)=\{a_{1}, a_{2}, a_{3}\}$.
$\overline{R(M)}(Y)=\{a_{1}, a_{2},\\ a_{3}, a_{5}\}$.
\end{example}

\begin{proposition}
\label{p3}
Let $M(U, \mathcal{I})$ be a general matroid. $R(M)$ is reflexive if and only if $\bigcup \mathcal{C}(M)=U$.
\end{proposition}
\begin{proof}
$(\Rightarrow)$: If $R(M)$ is a reflexive relation on $U$, then there exists $C\in \mathcal{C}(M)$ such that $\{x\}\subseteq C$ for all $x\in U$. Therefore $\bigcup \mathcal{C}(M)=U$.

$(\Leftarrow)$: If $\bigcup \mathcal{C}(M)=U$, then there exists $C\in \mathcal{C}(M)$ such that $y\in C$
for all $y\in U$, i.e., $(y, y)\in R(M)$.

This completes the proof.
\end{proof}

The above proposition shows that the sufficient and  necessary condition of $R(M)$ is a reflexive relation. In fact, $R(M)$ is serial and $R$
is reflexive are equivalent.

\begin{proposition}
\label{p4}
Let $M(U, \mathcal{I})$ be a general matroid. $R(M)$ is reflexive if and only if $R(M)$ is serial.
\end{proposition}
\begin{proof}
$(\Rightarrow)$: It is straightforward.

$(\Leftarrow)$: If $R(M)$ is serial, then $R(M)N(x)\neq\emptyset$ for all $x\in U$. Therefore $\bigcup \mathcal{C}(M)=U$.
It is easy to prove $R(M)$ is reflexive by Proposition~\ref{p3}.
\end{proof}

\begin{example}
\label{e7}
Let $G=(V$, $U)$ be the graph as shown in Fig.\ref{fig2}. Denote $\mathcal{C}=\{C\subseteq U\mid C$ does a cycle of $G\}$, i.e., $\mathcal{C}=\{\{a_{1}\}, \{a_{2}\}, \{a_{5}, a_{8}\}, \{a_{3}, a_{4}, a_{5}\}, \{a_{3}, a_{4}, a_{8}\}, \{a_{5}, a_{6},$\\$ a_{7}\}, \{a_{6}, a_{7}, a_{8}\},\{a_{3}, a_{4}, a_{6}, a_{7}\}\}$.
Then there exists a matroid $M=(U$, $\mathcal{I})$ such that $\mathcal{C}=\mathcal{C}(M)$, where $U=\{a_{1}, a_{2}, a_{3}, a_{4}, a_{5}, a_{6}, a_{7}, a_{8}\}$. Clearly, $R(M)=\{(a_{1}, a_{1}), (a_{2},\\ a_{2}), (a_{3}, a_{3}), (a_{3}, a_{4}), (a_{3}, a_{5}), (a_{3}, a_{6}), (a_{3}, a_{7}), (a_{3}, a_{8}), (a_{4}, a_{3}), (a_{4}, a_{4}), (a_{4}, a_{5}),\\ (a_{4}, a_{6}), (a_{4}, a_{7}), (a_{4}, a_{8}), (a_{5}, a_{3}), (a_{5}, a_{4}), (a_{5}, a_{5}), (a_{5}, a_{6}), (a_{5}, a_{7}), (a_{5}, a_{8}), (a_{6},\\ a_{3}), (a_{6}, a_{4}), (a_{6}, a_{5}), (a_{6}, a_{6}), (a_{6}, a_{7}), (a_{6}, a_{8}), (a_{7}, a_{3}), (a_{7}, a_{4}), (a_{7}, a_{5}), (a_{7}, a_{6}),\\ (a_{7}, a_{7}), (a_{7}, a_{8}), (a_{8}, a_{3}), (a_{8}, a_{4}), (a_{8}, a_{5}), (a_{8}, a_{6}), (a_{8}, a_{7}), (a_{8}, a_{8})\}$. It is easy to know $R(M)$ is an equivalence relation.

\begin{figure}
\begin{center}
  \includegraphics[width=3cm]{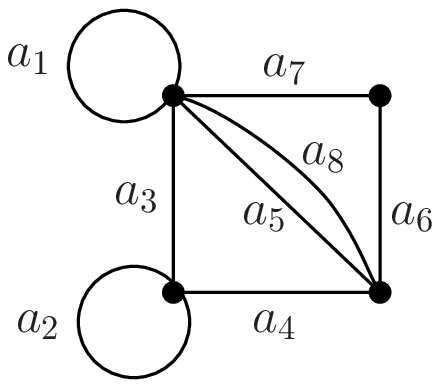}\\
  \caption{A graph}\label{fig2}
\end{center}
\end{figure}
\end{example}
\begin{proposition}
\label{p5}
Let $M(U, \mathcal{I})$ be a general matroid. For all $X\subseteq U$, $\overline{R(M)}(X)=cl_{M}(X)$ if and only if $\bigcup\mathcal{C}(M)=U$.
\end{proposition}
\begin{proof}
$(\Rightarrow)$: Since $cl_{M}(X)=X\bigcup\{u\in U: \exists C\in\mathcal{C}(M)\rightarrow u\in C\subseteq X\bigcup\{u\}\}=\bigcup\{C\in\mathcal{C}(M): C\bigcap X\neq\emptyset\}=\overline{R(M)}(X)$, then there exists $C\in \mathcal{C}(M)$ such that $x\in C$ for all $x\in U$, i.e., $\bigcup\mathcal{C}(M)=U$.

$(\Leftarrow)$: If $\bigcup\mathcal{C}(M)=U$, Then $R(M)$ is an equivalence relation and $\mathcal{C}(M)$ is a partition. Since
$\overline{R(M)}(X)=\bigcup\{C\in\mathcal{C}(M): C\bigcap X\neq\emptyset\}$ and $r_{M}(X)=|\{C\in\mathcal{C}(M): C\bigcap X\neq\emptyset\}|$.
If $y\in\overline{R(M)}(X)$, then $|\{C\in\mathcal{C}(M): C\bigcap X\neq\emptyset\}|\leq r_{M}(X\bigcup\{y\})\leq r_{M}(\overline{R(M)}(X))=|\{C\in\mathcal{C}(M): C\bigcap X\neq\emptyset\}|$. Therefore $y\in cl_{M}(X)$. If $y\notin\overline{R(M)}(X)$, obviously, $r_{M}(X\bigcup\{y\})=|\{C\in\mathcal{C}(M): C\bigcap X\neq\emptyset\}|+1$ holds. That is $y\notin cl_{M}(X)$. We have proved $\overline{R(M)}(X)=cl_{M}(X)$.
\end{proof}

Usually, relations on a universe $U$ can be represented by diagrams~\cite{RajagopalMason92Discrete}. In~\cite{ChenLi12AnApplication}, a new relation induced from a simple undirected graph is introduced and the authors have already proved the relation and simple graphs are one-one correspondence. Similarly, a new undirected graph introduced from $R(M)$ can be defined. For this purpose, some concepts of graph theory~\cite{Johnmurty76graph,Douglas02Introductionto} are presented as follows:

A graph is a pair $G=(U, E)$ consisting of a set $U$ of vertices and a set $E$ of edges such that $E\subseteq U\times U$. Two vertices are adjacent
if there is an edge that has them as ends. An isolated vertex is a vertex
not adjacent to any other vertex. The edges of a graph may be directed (asymmetric) or undirected (symmetric).
An undirected graph is one in which edges are symmetric. A graph is simple if
every edge links a unique pair of distinct vertices. A subgraph of
a graph $G$ is a graph whose vertices and edges are subsets of $G$.
The subgraph induced by a subset of vertices $K\subseteq U$ is called a vertex-induced subgraph of $G$, and denoted by $G_{k}$.
This subgraph has vertex set $K$, and its edge set $E'\subseteq E$ consists of those edges from
$E$ that have both their ends in $K$.

A path in a graph $G$ is a sequence $x_{1}, x_{2}, \ldots, x_{k}$ of distinct vertices such that $x_{i}x_{i+1}$ is an edge
of $G$ for $1\leq i\leq k-1$. Such a path is said to connect $x_{1}$ and $x_{k}$. A graph is connected if for every pair of distinct
vertices $x$ and $y$, there is a path connecting $x$ and $y$. In fact, every
finite graph $G$ can be partitioned into nonempty subgraphs $G(U_{1}, E_{1}), G(U_{2}, E_{2}), \ldots, G(U_{n}, E_{n})$ such that two vertices
$x$ and $y$ are connected if and only if both $x$ and $y$ belong to the same subgraph $G(U_{i}, E_{i})$. We call these subgraphs the components of $G$, and denote the number of components of $G$ by $\omega(G)$.

\begin{example}
\label{e8}
Let $G=(U, E)$ be a graph with $U=\{a, b, c, d, e\}$ and $E=\{e_{1}, e_{2}, e_{3}, e_{4}\}$. The diagram of $G$ is shown in Fig.~\ref{fig3}.
\begin{figure}
  \begin{center}
   \subfigure[$G$]
 {\label{fig3}
 \includegraphics[width=2.5cm]{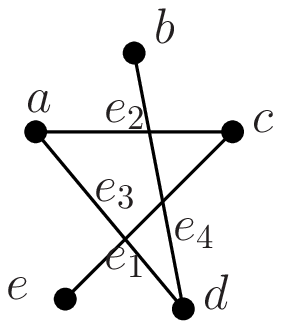}} \hspace{1cm}
 \subfigure[$G_{K}$]
 { \label{fig4}
 \includegraphics[width=2.5cm]{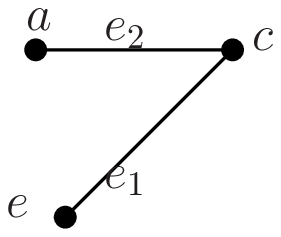}}
 \caption{The graph $G$ and the subgraph $G_{K}$.} 
  \end{center}
 \end{figure}

We see that $G$ is connected and has only one component. Let
$K =\{a, c, e\}$, the subgraph $G_{K}$
induced by the vertex set $K$ is shown in
Fig.~\ref{fig4}
\end{example}

We have already recalled some basic notions of graph theory. Now, a new undirected graph induced from a general matroid
is presented in the following definition.

\begin{definition}
Let $M(U, \mathcal{I})$ be a general matroid. We can define an undirected graph $G(M)$ as
follows: for all $x, y\in U$,
\begin{center}
$(x, y)\in E(G(M))\Leftrightarrow xR(M)y\wedge x\neq y$.
\end{center}
We say that the undirected graph $G(M)$ is introduced from $M$.
\end{definition}

\begin{example}(Continued from Example~\ref{e1})
\label{e9}
The relation induced from $M$ in Example~\ref{e1} is $R(M)=\{(a_{1}, a_{1}), (a_{1}, a_{2}), (a_{1}, a_{3}), (a_{2}, a_{1}), (a_{2}, a_{2}), (a_{2}, a_{3}), (a_{3}, a_{1}), (a_{3}, a_{2}), (a_{3},\\ a_{3}), (a_{5}, a_{5})\}$. Then the undirected graph $G(M)$ introduced from $M$ shows in Fig.~\ref{fig5}.
\begin{figure}
\begin{center}
  \includegraphics[width=3cm]{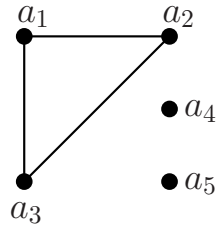}\\
  \caption{The graph introduced from $M$ }\label{fig5}
\end{center}
\end{figure}
\end{example}

The above example shows that the undirected graph induced from a general matroid is disconnected. In the following, the conditions under which the undirected graph induced from a general matroid is connected are presented.
\begin{proposition}
\label{p6}
Let $M(U, \mathcal{I})$ be a general matroid and $\bigcup\mathcal{C}(M)=U$. For all $C_{1}, C_{2}\in\mathcal{C}(M)$, if $C_{1}\bigcap C_{2}\neq\emptyset$, then $G(M)$ is connected.
\end{proposition}
\begin{proof}
Since $C_{1}\bigcap C_{2}\neq\emptyset$ for all $C_{1}, C_{2}\in\mathcal{C}(M)$, then $\{x\}\notin\mathcal{C}(M)$ and $|R(M)N(x)|>1$ for all $x\in U$. Therefore there is a path between any pair of distinct vertices and $G(M)$ is connected.
\end{proof}
\begin{example}
\label{e10}
Let $M(U, \mathcal{I})$ be a general matroid on $U=\{a_{1}, a_{2}, a_{3}, a_{4}\}$, where $\mathcal{C}(M)=\{\{a_{1}, a_{2}\}, \{a_{1}, a_{3}, a_{4}\}, \{a_{2}, a_{3}, a_{4}\}\}$. Then $R(M)=U\times U$ and the undirected graph $G(M)$ introduced from $M$ shows in Fig.~\ref{fig6}. It is clearly that $G(M)$ is conneced.
\begin{figure}
\begin{center}
  \includegraphics[width=4cm]{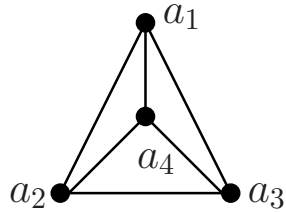}\\
  \caption{The graph introduced from $M$ }\label{fig6}
\end{center}
\end{figure}
\end{example}

\begin{proposition}
\label{p7}
Let $M(U, \mathcal{I})$ be a general matroid and $\bigcup\mathcal{C}(M)=U$. For all $\emptyset\neq X\subset U$, $G(M)$ is connected if and only if $X\neq \bigcup\{C\in\mathcal{C}(M): C\bigcap X\neq\emptyset\}$.
\end{proposition}
\begin{proof}
$(\Rightarrow)$: For any $\emptyset\neq X\subset U$, if $G(M)$ is connected, then there exist $x\in X$ and $y\in X^{C}$ such there
is an edge between $x$ and $y$. Then $y\in R(M)N(x)=\bigcup\{C\in\mathcal{C}(M): x\in C\}$. Hence $y\in R(M)N(x)$, this imlies there exists
$C'\in\mathcal{C}(M)$ such that $C'\bigcap X\neq\emptyset$ and $y\in C'$. Therefore $X\neq \bigcup\{C\in\mathcal{C}(M): C\bigcap X\neq\emptyset\}$ holds.

$(\Leftarrow)$: Suppose $G(M)$ is disconnected. Then there exist distinct vertices $x$ and $y$ such that there is not any path
connecting $x$ and $y$. Since $G(M)$ is disconnected, then there exists a connected component $G{M}_{1}=\{U_{1}, E_{1}\}$ of $G(M)$
such that $x\in U_{1}$ and $y\notin U_{1}$. Since $G(M)_{1}$ is connected, then $U_{1}=\bigcup\{C\in \mathcal{C}(M): C\bigcap U_{1}\neq\emptyset\}\subset U$, a contradiction to the assumption $U_{1}\neq \bigcup\{C\in\mathcal{C}(M): C\bigcap U_{1}\neq\emptyset\}$. Therefore
$G(M)$ is connected.
\end{proof}

\begin{example}(Continued from Example~\ref{e1} and Example~\ref{e10})
\label{e11}
In Example~\ref{e1}, since $\{a_{5}\}\subseteq \{a_{1}, a_{2}, a_{3}, a_{4}, a_{5}\}$ and $\bigcup\{C\in\mathcal{C}(M): C\bigcap\{a_{5}\}\neq\emptyset\}=\{a_{5}\}$. Thus according to Proposition~\ref{p7}, $G(M)$ is disconnected. In Example~\ref{e10}, by computing, for any $\emptyset\neq X\subset U$, $X\neq \bigcup\{C\in\mathcal{C}(M): C\bigcap X\neq\emptyset\}$ follows. Thus according to Proposition~\ref{p7}, $G(M)$ is connected.
\end{example}

The above results show that a symmetric and transitive relation or an undirected graph can be generalized
by a general matroid. Therefore, through the relation-based rough sets, the connectivity of undirected graphs
is examined from the viewpoint of relation-based rough sets. Similarly, the connectivity of a general matroid is also
can be investigated from the perspective of relation-based rough sets.

\section{The connectivity for matroids through relation-based rough sets}
\label{section3}
In this section, we study the connectivity for matroids by means of relation-based rough sets. For this purpose,
some definitions and properties of connected matroids are presented.
\begin{definition}(Connected component and connected matroid~\cite{Lai01Matroid,Oxley93Matroid})
\label{connectedmatroid}
Let $M$ be a matroid on $U$. For all $e_{1}, e_{2}\in U$, we define $e_{1}R e_{2}\Leftrightarrow e_{1}=e_{2}$ or there exists $C\in\mathcal{C}(M)$ such that $e_{1}, e_{2}\in C$. Clearly, $R$ is an equivalence relation on $U$. The $R-$equivalent classes are called the connected components of $M$.

If there is only one connected components of $M$, then we say $M$ is a connected matroid.
\end{definition}

The term "connected component" has been defined as the above definition for both graphs and matroids.
In other words, there exist closed relationships between connected graphs and connected matroids.

\begin{example}
\label{e12}
Let $M(U, \mathcal{I})$ be a matroid on $U=\{a, b, c, d\}$, where $\mathcal{I}(M)=\{X\subseteq U: |X|\leq 2\}$.
Then $\mathcal{C}(M)=\{\{a, b, c\}, \{a, b, d\}, \{a, c, d\}, \{b, c, d\}\}$. Clearly, it is easy to know that $M$ is a connected matroid by Definition~\ref{connectedmatroid}.
\end{example}

\begin{theorem}(~\cite{Lai01Matroid,Oxley93Matroid})
\label{theorem6}
The matroid $M$ is connected if and only if, for every pair of distinct elements of $U(M)$, there is a circuit containing both.
\end{theorem}

By the above theorem, a sufficient and necessary condition under which a matroid is connected is presented. Based on this theorem, some conditions under which a general matroid is connected can be obtained by means of relation-based rough sets.
\begin{proposition}
\label{p8}
Let $M=(U, \mathcal{I})$ be a general matroid. $M$ is connected if and only if $y\in R(M)N(x)$ for any $x, y\in U\wedge x\neq y$.
\end{proposition}
\begin{proof}
It is easy to prove this proposition by Definition~\ref{relation-matroid} and Theorem~\ref{theorem6}.
\end{proof}

\begin{example}
\label{e13}
Let $G=(V$, $U)$ be the graph as shown in Fig.\ref{fig7}. Denote $\mathcal{C}=\{C\subseteq U\mid C$ does a cycle of $G\}$, i.e., $\mathcal{C}=\{\{a_{1}, a_{2}, a_{5}\}, \{a_{3}, a_{4}, a_{5}\}, \{a_{1}, a_{2}, a_{3}, a_{4}\}\}$.
Then there exists a matroid $M=(U$, $\mathcal{I})$ such that $\mathcal{C}=\mathcal{C}(M)$, where $U=\{a_{1}, a_{2}, a_{3}, a_{4}, a_{5}\}$. Clearly, $R(M)=U\times U$. It is easy to know $M$ is a connected matroid.
\begin{figure}
\begin{center}
  \includegraphics[width=3cm]{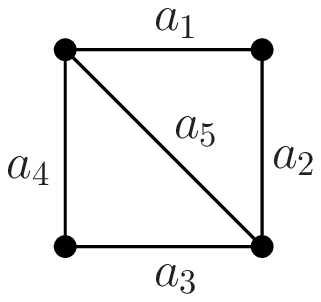}\\
  \caption{A graph}\label{fig7}
\end{center}
\end{figure}
\end{example}

\begin{proposition}
\label{p9}
Let $M=(U, \mathcal{I})$ be a general matroid. $M$ is disconnected if and only if there exists $\emptyset\neq X\subset U$ such that $ \overline{R(M)}(X)\subseteq X$.
\end{proposition}

\begin{proof}
$(\Rightarrow)$: If $M$ is disconnected, then there exists at leat two connected components of $M$. Suppose $X_{1}, X_{2}, \ldots, X_{s}(s\geq 2)$ is
all the connected components of $M$. Then $\emptyset\neq X_{1}\subset U$. Suppose $X=X_{1}$. If $\overline{R(M)}(X)\nsubseteq X$, then there exists $x\in \overline{R(M)}(X)$ such that $x\notin X$. This contradiction to $X_{1}$ is
a connected component of $M$.

$(\Leftarrow)$: Suppose $M$ is connected. For all $\emptyset\neq X\subset U$, then there exist $x\in X$ and $y\in X^{C}$ such
that $x\in R(M)N(y)$. Therefore $R(M)N(y)\bigcap X\neq\emptyset$, i.e., $y\in \overline{R(M)}(X)$. It is contradictory. Hence $M$ is a disconnected
matroid.

This completes the proof.
\end{proof}

The above proposition indicates that we may use the upper relation-based approximations to determine
whether a general matroid is connected or disconnected.
\begin{example} (Continued from Example~\ref{e13})
\label{e14}
Let $M(U, \mathcal{I})$ be a general matroid on $U=\{a_{1}, a_{2}, a_{3}, a_{4}, a_{5}\}$, where $\mathcal{C}(M)=\{\{a_{1}, a_{4}\}, \{a_{1}, a_{2}, a_{5}\}, \{a_{2}, a_{4}, a_{5}\}, \{a_{3}\}\}$. Then $R(M)=\{(a_{1}, a_{1}), (a_{1}, a_{2}), (a_{1}, a_{4}), (a_{1}, a_{5}), (a_{2}, a_{1}), (a_{2}, a_{2}), (a_{2}, a_{4}), (a_{2}, a_{5}), (a_{3},\\ a_{3}), (a_{4}, a_{1}), (a_{4}, a_{2}), (a_{4}, a_{4}), (a_{4}, a_{5}), (a_{5}, a_{1}), (a_{5}, a_{2}), (a_{5}, a_{4}), (a_{5}, a_{5})\}$. By computing, we have $\overline{R(M)}(\{a_{3}\})=\{a_{3}\}$. Therefore $M$ is a disconnected
matroid.

In Example~\ref{e13}, it is easy to find that $\overline{R(M)}(X)\nsubseteq X$ for any $\emptyset\neq X\subset U$. Hence the matroid $M$
represented in Example~\ref{e13} is connected.
\end{example}

Based on the above proposition, we provide a more effective approach to determine whether a general matroid is connected or disconnected in the following corollaries.
\begin{corollary}
Let $M=(U, \mathcal{I})$ be a general matroid. $M$ is connected if and only if for any $\emptyset\neq X\subset U$, there exists $x\in\overline{R(M)}(X)$ such that $x\notin X$.
\end{corollary}

\begin{corollary}
Let $M=(U, \mathcal{I})$ be a general matroid. $M$ is not connected if and only if for any $\emptyset\neq X\subset U, \overline{R(M)}(X)\bigcap X^{C}\neq\emptyset$.
\end{corollary}

In fact, the above two corollaries are the generalizations of Proposition~\ref{p9} and the proofs of them are simple.

According to Proposition~\ref{p3}, we know that $R(M)$ is an equivalence relation on $U$ if $\bigcup\mathcal{C}(M)=U$.
Then we can study the conditions under which a general matroid is connected if and only if $\bigcup\mathcal{C}(M)=U$.
\begin{proposition}
\label{p10}
Let $M=(U, \mathcal{I})$ be a general matroid and $\bigcup\mathcal{C}(M)=U$. $M$ is connected if and only if $\overline{R(M)}(X)\neq X$ for all $\emptyset\neq X\subset U$.
\end{proposition}
\begin{proof}
Since $\bigcup\mathcal{C}(M)=U$, then $R(M)$ is an equivalence relation on $U$.

$(\Rightarrow)$: It is straightforward.

$(\Leftarrow)$: If $M$ is disconnected, then there exists $\emptyset\neq X\subset U$ such that $\overline{R(M)}(X)\subseteq X$. Since
$R(M)$ is an equivalence relation on $U$, then $X\subseteq\overline{R(M)}(X)$. Hence $\overline{R(M)}(X)=X$. It is contradictory. Therefore
$\overline{R(M)}(X)\neq X$ for all $\emptyset\neq X\subset U$.
\end{proof}

Based on the Proposition~\ref{p7} and Proposition~\ref{p10}, the relationships between $G(M)$ and $M$ are established as following proposition shown.
\begin{proposition}
\label{p11}
Let $M=(U, \mathcal{I})$ be a general matroid and $\bigcup\mathcal{C}(M)=U$. $M$ is connected if and only if $G(M)$ is connected.
\end{proposition}
\begin{proof}
It is easy to prove this proposition by Proposition~\ref{p7} and Proposition~\ref{p10}.
\end{proof}

\begin{example}(Continued from Example~\ref{e13})
It is easy to know that the undirected graph
$G(M)$ introduced from $M$ is shown in Fig.\ref{fig8}.
Clearly, $G(M)$ is a connected graph. In Example~\ref{e14}, we have already known the matroid $M$ represented in Example~\ref{e13} is connected.
\begin{figure}
\begin{center}
  \includegraphics[width=3cm]{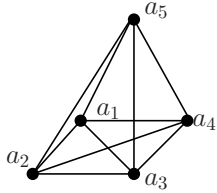}\\
  \caption{The undirected graph introduced from $M$}\label{fig8}
\end{center}
\end{figure}
\end{example}

In the following, some properties of connected matroid can be investigated from the viewpoint of relation-based rough sets.
\begin{proposition}
\label{p12}
Let $M=(U, \mathcal{I})$ be a general matroid. If $\bigcup\mathcal{C}(M)=U$, then
\begin{center}
$\overline{R(M)}(R(M)N(x))=R(M)N(x)$.
\end{center}
\end{proposition}
\begin{proof}
If $\bigcup\mathcal{C}(M)=U$, then $\{R(M)N(x): x\in U\}$ is a partition of $U$. Thus $\overline{R(M)}(R(M)N(x))=R(M)N(x)$.
\end{proof}

The above proposition shows an important properties of classical upper rough approximations.
In other words, the $R-$equivalent class is a $R-$precise set.

This section points out some conditions under which a general matroid is connected by means of relation-based rough sets.
These results for connected matroid present new perspectives to study connectivity for matroids.
\section{Conclusions}
\label{section4}
This paper studies the connectivity for matroids by means of relation-based rough sets. First, we defined a symmetric and transitive relation $R(M)$ through a general matroid $M$. Second, based on the definition of $R(M)$, some properties such as reflexive, serial and so on are investigated in Section~\ref{section2}. Third, an undirected graph $G(M)$ was introduced from $M$ and the connectivity of $G(M)$ was explored.
Fourth, some conditions under which a general matroid $M$ is connected were presented from the viewpoint of relation-based rough sets. Finally, we have proved that the connectivity of $M$ and $G(M)$ is equivalent if $\mathcal{C}(M)$ is a covering of $U(M)$.

\section{Acknowledgments}
This work is in part supported by the National Science Foundation of China
under Grant Nos. 61170128, 61379049, and 61379089,
the Natural Science Foundation of Fujian Province, China under Grant No. 2012J01294,
the Fujian Province Foundation of Higher Education under Grant No. JK2012028, and
the Education Department of Fujian Province under Grant No. JA12222.

\end{document}